\documentclass{article}

\usepackage[preprint]{neurips_2023}
\usepackage{amsthm}
\usepackage{amsmath}
\usepackage{amsfonts}

\theoremstyle{plain}
\newtheorem{theorem}{Theorem}[section]
\newtheorem{corollary}[theorem]{Corollary}
\newtheorem{assumption}[theorem]{Assumption}
\newtheorem{lemma}[theorem]{Lemma}
\theoremstyle{definition}
\newtheorem{definition}[theorem]{Definition}

\DeclareMathOperator*{\argmax}{arg\,max}

\title{Analysis of Value Iteration Through Absolute Probability Sequences}

\author{%
  Arsenii Mustafin$^*$  \\
  Department of Computer Science\\
  Boston University\\
  Boston, MA 02215, USA \\
  \texttt{aam@bu.edu} \\
  \And
  Sebastien Colla$^*$ \\
  ICTEAM Institute \\
  UCLouvain, 1348 Louvain-la-Neuve, Belgium \\
  sebastien.colla@uclouvain.be \\
   \AND
   Alex Olshevsky \\
   Department of ECE \\
   Boston University, Boston, MA 02215, USA \\
   \texttt{alexols@bu.edu} \\
   \And
   Ioannis Ch. Paschalidis \\
   Department of ECE \\
   Boston University, Boston, MA 02215, USA \\
   \texttt{yannisp@bu.edu} \\
}
\begin{document}

\maketitle
\begin{abstract}
Value Iteration is a widely used algorithm for solving Markov Decision Processes (MDPs). While previous studies have extensively analyzed its convergence properties, they primarily focus on convergence with respect to the infinity norm. In this work, we use absolute probability sequences to develop a new line of analysis and examine the algorithm’s convergence in terms of the $L^2$ norm, offering a new perspective on its behavior and performance.
\end{abstract}

\section{Introduction} \label{sec:intro}

Markov Decision Processes (MDPs) are a foundational mathematical framework for sequential decision-making and have become a cornerstone of modern AI. They can be viewed as an extension of standard Markov Processes (MPs) designed to model sequential decision-making. MDPs were first introduced in the late 1950s, along with a key algorithm for solving them—Value Iteration \citep{bellmandp}. The algorithm's convergence in terms of the infinity norm, at a rate determined by the discount factor $\gamma$ (introduced below), was established in \citep{howard1960dynamic}, where it was also shown that this convergence rate is achievable. Subsequent analyses have focused on particular cases \citep{puterman1990markov, feinberg2014value}.

In this work, we continue this line of research by analyzing the convergence of the Value Iteration algorithm. The key novelty of our approach lies in the analytical framework: we use absolute probability sequences to establish the convergence of Value Iteration in terms of the L² norm. Additionally, we characterize the convergence rate in the case where a unique optimal policy induces a strongly connected graph.

\section{Mathematical setting} \label{sec:setting}

We follow the general discounted infinite-horizon MDP setting from \cite{puterman2014markov} and \cite{sutton2018reinforcement}. We consider an MDP with a finite set of states, denoted as $\mathcal{S}$, with $|\mathcal{S}| = n$, and a finite set of actions, denoted as $\mathcal{A}$, with $|\mathcal{A}| = m$. We assume that each action in this set is available in every state. The expected reward for taking action $a$ in state $s$ is denoted as $r(s,a)$, and the transition probability distribution over the state set is given by $P(\cdot | s,a)$. The discount factor of the MDP is denoted as $\gamma \in (0,1)$.  

A policy $\pi$ defines a choice rule for selecting action $a_s$ in each state $s$, written as $\pi(s) = a_s$. If a policy selects a single action with probability 1, it is called a deterministic policy. In this paper, we consider only deterministic policies. We say that a policy $\pi$ is implied by a values $V$ if  

\begin{equation*}
\pi(s) = \argmax_{a \in \mathcal{A}} \left( r(s,a) + \gamma \sum_{s'} P(s'|s,a)V(s') \right).
\end{equation*}

Given a policy $\pi$, we define its value in state $s$, denoted as $V^\pi(s)$, to be the expected sum of the discounted rewards over an infinite trajectory starting from state $s$:
\begin{equation*}
V^\pi (s) = \mathbb{E} \left[ \sum_{t=1}^{\infty} \gamma^t r_t \right],    
\end{equation*} 
where $r_t$ is the reward received at time step $t$, starting from state $s$, and actions in each state are chosen according to policy $\pi$.

We are interested in finding the optimal policy $\pi^*$ such that:
\[
V^{\pi^*} (s) \geq V^\pi(s), \quad \forall \pi, s.
\]

The value vector produced by $\pi^*$ is denoted as $V^*$ and satisfies:
\[
V^*(s) = r(s,\pi^*(s)) + \gamma \sum_{s'} P(s'|s,\pi^*(s)) V^*(s').
\]

The Value Iteration algorithm solves this problem by computing a sequence $\{V_t\}$ with an arbitrarily chosen initial value $V_0 \in \mathbb{R}^n$. The sequence is generated by the update:

\begin{equation} \label{eq:vi_update}
V_{t+1} = (1-\alpha)V_t + \alpha \max_{a \in \mathcal{A}} \left( r(s,a) + \gamma \sum_{s'}P(s'|s,a)V_t(s') \right),
\end{equation}
where $\alpha \in (0,1)$ is a chosen learning rate. Here and throughout, the subscript $\square_t$ denotes quantities obtained after $t$ updates using \eqref{eq:vi_update}. In this case. vector or matrix elements, indices are given in square brackets ($[i]$ or $[i,j]$); otherwise when timestep is not important, they are written as subscripts ($\square_i$ or $\square_{i,j}$).

It was shown in \cite{howard1960dynamic} that the update \eqref{eq:vi_update} is a contraction, with $V^*$ as its unique stationary point, implying that the sequence $\{V_t\}$ converges to $V^*$ in terms of the infinity norm.

In this paper, we propose a different proof technique and show that this sequence converges to $V^*$ in terms of the weighted $L^2$ norm. To avoid the worst-case scenario, we make the following key assumption:

\begin{assumption} \label{as:ergodicity}
The unique optimal policy $\pi^*$ induces an irreducible and aperiodic Markov process.
\end{assumption}

\section{Our Approach}
Our approach focuses on the analysis of error vectors:
\[
e_t = V_t - V^*.
\]
The first simple observation we make is that any consensus vector $e_t$, \textit{i.e.}, $e_t = c\mathbf{1}$, implies the optimal policy. This allows us to focus on analyzing the convergence of $e_t$ to consensus, which does not necessarily imply convergence to zero.

The update \eqref{eq:vi_update} can be written in vector form as:
\begin{equation*}
V_{t+1} = (1-\alpha)V_t + \alpha (r_t + \gamma P_t V_t),
\end{equation*}
where the choices made by the maximization operator are encoded in the reward vector $r_t$ and the transition matrix $P_t$. We denote the rewards and transitions associated with the optimal value vector $V^*$ as $r^*$ and $P^*$. Using the maximization property of these choices, we derive the following bounds on $e_t$. 

\textbf{Upper bound:}
\begin{align}
e_{t+1} &= V_{t+1} - V^* = (1-\alpha)e_t + 
\alpha \left[r_t + \gamma P_t V_t - r^* - \gamma P^* V^* \right] \nonumber \\
&\leq (1-\alpha)e_t +  \alpha \left[r_t + \gamma P_t V_t - r_t - \gamma P_t V^*\right] =
\left[ (1-\alpha) I + \alpha\gamma P_t\right]  e_t. \label{eq:e_upper_bound}
\end{align}

\textbf{Lower bound:}
\begin{align}
e_{t+1} &= V_{t+1} - V^* = (1-\alpha)e_t + 
\alpha \left[r_t + \gamma P_t V_t - r^* - \gamma P^* V^* \right] \nonumber \\
&\geq (1-\alpha)e_t +  \alpha \left[r^* + \gamma P^* V^* - r^* - \gamma P^* V^*\right] =
\left[ (1-\alpha) I + \alpha\gamma P^*\right]  e_t. \label{eq:e_lower_bound}
\end{align}

We normalize the expression in \eqref{eq:e_upper_bound} to obtain a stochastic matrix:
\begin{equation*}
 \left[ (1-\alpha) I + \gamma P_t\right] = [(1-\alpha) + \alpha\gamma]P_{t\alpha} = \gamma_\alpha P_{t\alpha},
\end{equation*}
where $P_{t\alpha}$ is a stochastic matrix with diagonal entries of at least $(1-\alpha)/((1-\alpha) + \alpha\gamma) > (1-\alpha)$.

Normalizing the matrix in the lower bound \eqref{eq:e_upper_bound} in the same way and combining both bounds, we obtain:
\begin{equation} \label{eq:e_double_bound}
\gamma_\alpha P^*_\alpha e_t\le e_{t+1} \le \gamma_\alpha P_{t\alpha} e_t.
\end{equation}
This inequality implies that each element of $e_{t+1}$ can be represented as a convex combination of the left- and right-hand side expressions. By introducing a diagonal matrix $D_t$ to combine these expressions, the error vector dynamics can be expressed as:

\begin{equation} \label{eq:error_update}
e_{t+1} =\gamma_\alpha \left[ D_t P^*_\alpha + (I - D_t) P_{t\alpha} \right]  e_t=\gamma_\alpha M_t e_t.
\end{equation}

By construction, all matrices $M_t$ are row stochastic with diagonal entries of at least $1-\alpha$. Additionally, by Assumption \ref{as:ergodicity}, the underlying graph of $M_t$ is strongly connected. The stochasticity of $M_t$ implies the existence of a reversible Markov chain, which serves as our main analysis tool:

\begin{definition}
Let $\{M_t\}$ be a sequence of row stochastic matrices and $\{p_t \}$ a sequence of probability distribution vectors. Then, $\{p_t \}$ is called an \textit{absolute probability sequence} of $\{M_t\}$ if
\begin{equation} \label{eq:p_dynamics}
p_{t+1}^T M_t = p_t^T, \quad \forall t \geq 0.
\end{equation}
\end{definition}

The existence of an absolute probability sequence for a sequence of stochastic matrices $\{M_t\}$ was established by Kolmogorov \citep{kolmogoroff1936theorie}. This sequence allows us to define a corresponding norm, which will be used to show convergence:

\begin{definition}
Let $\{p_t \}$ be an absolute probability sequence. We define the corresponding $\{p_t \}$-weighted scalar product between vectors $x$ and $y \in \mathbb{R}^n$ as
\begin{equation*}
\langle x,y \rangle_{p_t} = \sum_{i=1}^n p_t[i]x[i]y[i],
\end{equation*}
and the corresponding $p_t$-weighted squared norm as:
\begin{equation*}
||x||^2_{p_t} = \sum_{i=1}^n p_t[i] x^2[i].
\end{equation*}
\end{definition}

Now, we decompose the error vector $e_t$ into two components: one in the consensus subspace and the other in its orthogonal complement, with respect to this new scalar product:
\begin{equation} \label{eq:decomposition}
e_t = c_t \mathbf{1} + \Delta_t,\, \text{where } \langle \mathbf{1}, \Delta_t \rangle_{p_t} =
p_t^T\Delta_t = 0.
\end{equation}
We then have that
\begin{equation*}
c_t = \frac{ \langle e_t, \mathbf{1} \rangle_{p_t}}{ \langle \mathbf{1}, \mathbf{1} \rangle_{p_t}},
\end{equation*}
which corresponds to the $p_t$-weighted average of $e_t$, providing insight into the choice of $c_t$.

\section{Main result}

\begin{theorem} \label{thm:main}
Let $e_t$ be a bounded initial error vector that is not at consensus (\textit{i.e.}, $\Delta_t \ne 0$), $\{M_t\}$ a sequence of row-stochastic matrices, $\gamma \in (0,1)$ a discount factor, and $\{p_t\}$ an absolute probability sequence for $\{M_t\}$. If, for any $t$, the matrix $M_t$ has diagonal entries of at least $(1-\alpha) \in (0,1)$ and a strongly connected underlying graph, then the sequence $e_t$ generated by Update \ref{eq:vi_update} satisfies
\begin{equation}
|| \Delta_{t+1}||_{p_{t+1}}^2 \leq \gamma_\alpha^2(1-R_t)|| \Delta_{t}||_{p_t}^2,
\end{equation}
where $R_t > 0$ for any $t \geq 0$.
\end{theorem}

This theorem implies the following corollary on the convergence of the Value Iteration algorithm with a rate higher than $\gamma$, which constitutes the main result of this paper.

\begin{corollary} \label{cor:main_result}
Suppose an MDP $\mathcal{M}$ and let $\{V_t\}$ be a sequence produced by the Value Iteration algorithm, starting with values $V_0 = e_0 + V^*$. Then, if Assumption \ref{as:ergodicity} holds, Value Iteration induces a sequence of row-stochastic matrices $\{M_t\}$ with diagonal entries $(1-\alpha) \in (0,1)$ and an absolute probability sequence $\{p_t\}$. The resulting error vector sequence $e_t = c_t \mathbf{1} + \Delta_t$ satisfies:
\begin{equation} \label{eq:main_result}
|| \Delta_{T}||_{p_T}^2 \leq (\gamma_\alpha^2(1-\lambda))^T || \Delta_{0}||_{p_0}^2,    
\end{equation}
where $0 < \lambda \leq R_t$ for all $0 \leq t < T$.
\end{corollary}

\section{Proofs}
This proof uses several standard lemmas that we state and prove here for completeness.
\begin{definition}{Laplacian matrix}
A matrix $L \in \mathbb{R}^{n \times n}$ is a Laplacian matrix if it is a symmetric
matrix with zero row sums and non-positive off-diagonal entries.
\end{definition}
\begin{lemma} \label{lem:lapl_matrix}
If $L \in \mathbb{R}^{n \times n}$ is a Laplacian matrix, then, for all $x \in \mathbb{R}^n$ ,
\begin{equation} \label{eq:laplacian_matrix_prop}
x^TLx = -\frac{1}{2} \sum_i \sum_{j \ne i} L[i,j] (x_i - x_j)^2.
\end{equation}
\end{lemma}
\begin{proof}
We expand the sum in Equation \ref{eq:laplacian_matrix_prop}:
\begin{align*}
\sum_i \sum_{j \ne i} L[i,j] (x_i - x_j)^2 &= \sum_{i,j}  L[i,j] (x_i - x_j)^2 = 
\sum_{i,j} L[i,j] x_i^2 - 2L[i,j] x_i x_j + L[i,j] x_j^2 = \\
&=-2\sum_{i,j} L[i,j] x_i x_j = -2x^tLx, \\
\end{align*}
where the third equality uses that $\sum_j L[i,j]=\sum_i L[i,j]=0$
\end{proof}

\begin{lemma} \label{lem:prob_dist}
Let $p \in \mathbb{R}^n_+$ be a non negative vector with sum of elements equal to $1$. Then, for any $x \in \mathbb{R}^n$
\begin{equation*}
\left( \sum_i p[i] x_i \right)^2 = \sum_i p[i] x_i^2  - \frac{1}{2} \sum_{i,j} p[i] p[j] (x_i - x_j)^2.
\end{equation*}
\end{lemma}
\begin{proof}
After rearranging terms we can write
\begin{equation*}
\sum_i p[i] x_i^2 - \left( \sum_i p[i] x_i \right)^2 = x^T (\text{diag}(p) - pp^T) x
\end{equation*}
The matrix $(\text{diag}(p) - pp^T)$ is a Laplacian matrix since it is symmetric, has rows aadding to $0$ and non-positive off-diagonal entries. We then can apply Lemma \ref{lem:lapl_matrix}:
\begin{equation*}
x^T (\text{diag}(p) - pp^T) x = -\frac{1}{2} \sum_i \sum_{i\ne j} (-p[i] p[j]) (x_i - x_j)^2 = \frac{1}{2} \sum_i \sum_{i\ne j} p[i] p[j]  (x_i - x_j)^2
\end{equation*}
\end{proof}
The following lemma starts the analysis of the convergence of sequence $\{ \Delta_t \}$.

\begin{lemma} \label{lem:c_delta_dynamics}
Let $e_t$ be the sequence of errors described by the update \ref{eq:error_update}. Then, its decomposition \ref{eq:decomposition} satisfies
\begin{align}
c_{t+1} &= \gamma_\alpha c_t \label{eq:c_dynamics} \\
\Delta_{t+1} &= \gamma_\alpha M_t \Delta_t \label{eq:delta_dynamics}
\end{align}
\end{lemma}
\begin{proof}
To obtain \ref{eq:c_dynamics} we need to multiply Equation \ref{eq:error_update} by $p_{t+1}^T$"
\begin{align*}
p_{t+1}^T e_{t+1} &= \gamma_\alpha p_{t+1}^T M_t e_{t} = \gamma_\alpha p_t^T e_t \implies
c_{t+1} = \gamma_\alpha c_t,
\end{align*}
where the equality in the first line is obtained by using the definition of $p_t$ \ref{eq:p_dynamics} and the equality in the second line by definition of $c_t$. 
Then, this result can be used to obtain \ref{eq:delta_dynamics}:
\begin{align*}
\Delta_{t+1} &= e_{t+1} - c_{t+1} \mathbf{1} = \gamma_\alpha M_t e_t - \gamma_\alpha c_t \mathbf{1}\\
    &= \gamma_\alpha M_t (e_t - c_t \mathbf{1}) = \gamma_\alpha M_t \Delta_t,
\end{align*}
where the second equality uses stochasticity of $M_t$, $M_t \mathbf{1}= \mathbf{1} $
\end{proof}
We are now ready to proof the main theorem:
\begin{proof}[Proof of Theorem \ref{thm:main}]
\begin{align*}
||\Delta_{t+1}||^2_{t+1} &= \sum_i p_{t+1}[i]\Delta_{t+1}^2[i]\\
\text{(using Lemma \ref{lem:c_delta_dynamics})} &= \sum_i p_{t+1}[i]\left(\gamma_\alpha \sum_j M_t[i,j]\Delta_t[j]\right)^2 \\
\text{(using Lemma \ref{lem:prob_dist})} &= \gamma_\alpha^2 \sum_i p_{t+1}[i] 
\left( \sum_j M_t[i,j] \Delta_t^2[j] - \frac{1}{2} \sum_{k,l} M_t[i,k]M_t[i,l] (\Delta_t[k] - \Delta_t[l])^2 \right) \\
&= \gamma_\alpha^2 \left(\sum_j \left( \sum_i p_{t+1}[i]M_t[i,j]\right)\Delta_t^2[j] - \frac{1}{2} \sum_{i,k,l} p_{t+1}[i] M_t[i,k]M_t[i,l] (\Delta_t[k] - \Delta_t[l])^2  \right) \\
\text{(by \ref{eq:p_dynamics})} &= \gamma_\alpha^2 \left( \sum_j p_t[j] \Delta_t^2[j]
- \frac{1}{2} \sum_{i,k,l} M_t[i,k]M_t[i,l] (\Delta_t[k] - \Delta_t[l])^2 \right) \\
&=\gamma_\alpha^2 \left(||\Delta_t||^2_{p_t} - \frac{1}{2} \sum_{i,k,l} M_t[i,k]M_t[i,l] (\Delta_t[k] - \Delta_t[l])^2   \right) \\
&=\gamma_\alpha^2 (1 - R_t) ||\Delta_t||^2_{p_t},
\end{align*}
where
\begin{equation} \label{eq:R_definition}
R_t = \frac{ \frac{1}{2} \sum_{i,k,l} M_t[i,k]M_t[i,l] (\Delta_t[k] - \Delta_t[l])^2}{||\Delta_t||^2_{p_t}}.
\end{equation}
To conclude the proof of the theorem, we need to show that $R_t$ is positive. First, let's define a minimum entry of $p_t$ as $p_{\rm min} > 0$ and $\epsilon$ to be the smallest non-zero entry of $M_t$. Recalling that diagonal entries of M are larger than $1-\alpha > 0$, we obtain
\begin{equation*}
R_t \ge \frac{\frac{1}{2} \sum_{i,k,l} p_{\rm min}(1-\alpha)\epsilon (\Delta_t[k] - \Delta_t[l])^2}{||\Delta_t||^2_{p_t}},
\end{equation*}
because for each row $i$ of $M_t$ we have a diagonal entry of at least $1-\alpha$ and at least one other non-zero entry, which is then not smaller than $\epsilon$. All the off-diagonal entries of row $i$ cannot be zero otherwise the row sum would not be $1$. Then, we can define $q=(1-\alpha) 
\epsilon p_{\rm min} > 0$ and write
\begin{equation*}
R_t \ge \frac{\frac{q}{2}\sum_{i,k,l} (\Delta_t[k] - \Delta_t[l])^2}{||\Delta_t||^2_{p_t}}
\end{equation*}
Then recall, that $p_t^T\Delta_t=0$, which means that all entries of the vector $\Delta_t$ can be equal only if $\Delta_t[i] = 0$ for all $i$. Therefore, $R_t=0$ implies that $\Delta_t=0$, which corresponds to a stationary point of the update \ref{eq:delta_dynamics} and was excluded by the theorem assumption. This observation concludes the proof of the theorem.
\end{proof}

\section{Characterization of the convergence rate}
Let us define the following Laplacian matrix:
\begin{equation*}
L_t[k,l]=\begin{cases}
\begin{aligned}
& -\sum_i p_{t+1}[i] M_t[i,k]M_t[i,l], &k \ne l\\
& \sum_i \sum_{k \ne l} p_{t+1}[i] M_t[i,k]M_t[i,l], &k=l\\
\end{aligned}
\end{cases}\end{equation*}
Or in the matrix notation:
\begin{equation*}
L_t = -M_t^T\Pi M_t + \text{diag}(M_t^T\Pi M_t \mathbf{1}),
\end{equation*}
where $\Pi_t = \text{diag}(p_t)$. Matrix $L_t$ is symmetric. By definition of its diagonal entries, its rows sum to zero. By non-negativity of $M$ and $p$, the off-diagonal entries are well non-positive. The matrix $L_t$ is thus well a Laplacian matrix.

Using these definitions, the term \ref{eq:R_definition} can be written as a generalized Rayleigh quotient:
\begin{equation} \label{eq:R_rayleigh}
R_t = \frac{\Delta^T L_t \Delta}{ \Delta^T \Pi_t \Delta}.
\end{equation}
The generalized Rayleigh quotient is presented in \cite{Pattabhiraman_1974}.
For clarity, below we do not use iteration indices $t$. the bounds on the generalized Rayleigh quotient are related to the generalized eigenvalue problem:
\begin{equation} \label{eq:gen_ev}
Lv = \lambda \Pi v.
\end{equation}
The matrix $\Pi$ is diagonal and therefore its invert $\Pi^{-1}$ exist sand is diagonal, as well as its factorization $\Pi=\Pi^{1/2}\Pi^{1/2}$. Thus, the pairs of eigenvalues and eigenvectors $(\lambda_i,v_i)$ associated with the problem \ref{eq:gen_ev} can be seen as the classical pairs of eigenvalues and eigenvectors associated with the matrix $\Pi^{-1}L$. One of this pairs is $(0,1)$, by definition of $L$. We can also relate these pairs $(\lambda_i, v_i)$ of eigenvalues and eigenvectors to the ones of the symmetric matrix $\Tilde{L}=\Pi^{-1/2}L \Pi^{-1/2}$, which are $(\lambda_i,\Pi^{1/2}v_i)$.
\begin{equation*}
Lv = \lambda \Pi v \Leftrightarrow \Pi^{-1}Lv = \lambda v \Leftrightarrow 
(\Pi^{-1/2}L \Pi^{-1/2})(\Pi^{1/2}v) = \lambda (\Pi^{1/2}v).
\end{equation*}

Since $\Tilde{L}$ is symmetric, its matrix of eigenvectors $\Pi^{1/2}v$ is orthogonal, in the Euclidean scalar product $(\Pi^{1/2}v)^T(\Pi^{1/2}v)=I$, meaning that the matrix $V$ of eigenvectors of $\Pi^{-1}L$ are orthogonal, with respect to the $p$-weighted scalar product $V^T\Pi V=I$. Also, we have that $\Pi^{-1/2}$ is diagonal and then matrices $\Tilde{L}$ and $L$ are congruent, meaning that they have the same number of positive, zero and negative eigenvalues.

Since $V$ contains $n$ orthogonal vectors, we can decompose any vector $\Delta \in \mathbb{R}^n$ into a combination of such vectors:
\begin{equation*}
\Delta = \sigma_i \mathbf{1} + \sum_{i=2}^n \sigma_i v_i
\end{equation*}
Let us plug this into \ref{eq:R_rayleigh}
\begin{align*}
\frac{\Delta^TL\Delta}{\Delta^T \Pi \Delta} = \frac{\Delta^T L(\sigma_i \mathbf{1} + \sum_{i=2}^n \sigma_i v_i)}{\Delta^T \Pi(\sigma_i \mathbf{1} + \sum_{i=2}^n \sigma_i v_i)} &=
\frac{\Delta^T L(\sum_{i=2}^n \sigma_i v_i)}{\Delta^T \Pi(\sum_{i=2}^n \sigma_i v_i)} \\
\text{(Using \ref{eq:gen_ev})} &=\frac{\Delta^T(\sum_{i=2}^n \sigma_i \lambda_i \Pi v_i)}{\Delta^T \Pi(\sum_{i=2}^n \sigma_i v_i)} \\
&=\frac{(\sum_{j=1}^n \sigma_j v_j)^T \Pi (\sum_{i=2}^n \sigma_i \lambda_i v_i)}{(\sum_{j=1}^n \sigma_j v_j)^T \Pi (\sum_{i=2}^n \sigma_i v_i)} \\
&=\frac{\sum_{j=1}^n \sum_{i=2}^n \sigma_j \sigma_i \lambda_i v_j^T \Pi v_i}{\sum_{j=1}^n \sum_{i=2}^n \sigma_j \sigma_i v_j^T \Pi v_i} \\
\text{(using orthogonality of $V$)} &=
\frac{\sum_{i=2}^n \sigma_i^2 \lambda_i^2 \lambda_i}{\sum_{i=2}^n \sigma_i^2} \ge \lambda_2(\Tilde{L}),\\
\end{align*} 
where the first line uses the facts $L \mathbf{1} = 0$ and $\Delta^T\Pi \mathbf{1} = 0$. The value $\lambda_2(\Tilde{L})$ denotes the second smallest eigenvalue of $\Tilde{L}$, which is our lower bound of the Raleigh quotient:
\begin{equation*}
R = \frac{\Delta^T L \Delta}{\Delta^T \Pi \Delta} \ge \lambda_2(\Tilde{L}).
\end{equation*}
To pursue the characterization of the bound, we link$\lambda_2(\Tilde{L})$ and $\lambda_2 (L)$ in the following Lemma.
\begin{lemma} \label{lem:lambda_bound}
Let $\Pi = \text{diag}(p)$ and $p_{\rm max} = \max_i p_i \in (0, 1]$. We consider a Laplacian matrix $L$ and a modified matrix $\Tilde{L}=\Pi^{-1/2}L \Pi^{-1/2}$, with $\lambda_2(L)$ and $\lambda_2 (\Tilde{L})$ their respective second smallest eigenvalues. Than,
\begin{equation*}
\lambda_2 (\Tilde{L}) \ge p_{\rm max}^{-1} \lambda_2 (L).
\end{equation*}
\end{lemma}
\begin{proof}
Both eigenvalues can be characterized as solutions of the following optimization problems
\begin{align}
\lambda_2 (\Tilde{L}) =& \min_x x^T \Tilde{L} x \label{eq:optim_tL} \\
&\text{s.t. } ||x||_2 = 1 \text{ and } x^T\mathbf{1}=0. \nonumber
\end{align}
and 
\begin{align}
\lambda_2 (L) =& \min_y y^T L y \label{eq:optim_L} \\
&\text{s.t. } ||y||_2 = 1 \text{ and } y^T \mathbf{1}=0. \nonumber
\end{align}
Let $y = \Pi^{-1/2}x$ and $x = \Pi^{1/2}y$. Then, we have that $x^T \Tilde{L} x = y^T L y$. We apply this variable replacement to \ref{eq:optim_tL}:
\begin{align}
\lambda_2 (\Tilde{L}) =& \min_y y^T L y \label{eq:optim_tL_replaced} \\
&\text{s.t. } ||\Pi^{1/2}y||_2^2 = 1 \text{ and } (\Pi^{1/2}y)^T\mathbf{1}=0. \nonumber
\end{align}
We also have that $1 = ||\Pi^{1/2}y||_2^2 \le p_{\rm max} ||y||_2^2$, meaning that $||y||_2 \le p_{\rm max}^{-1/2}$. Let's scale the problem \ref{eq:optim_L} accordingly:
\begin{align}
p_{\rm max}^{-1} \lambda_2 (L) =& \min_y y^T L y \label{eq:optim_L_scaled} \\
&\text{s.t. } ||y||_2 = p_{\rm max}^{-1/2} \text{ and } y^T \mathbf{1}=0. \nonumber
\end{align}
Finally, it is left to show that any feasible solution of \ref{eq:optim_L_scaled}, denoted $y_1$, we can build a feasible solution of \ref{eq:optim_tL_replaced} as $dy_1$ that achieves a higher objective values. Let us verify that $dy_1$ is a feasible solution of \ref{eq:optim_tL_replaced}:
\begin{itemize}
\item $\Pi^{1/2}d y_1^T\mathbf{1}=0$ since $y_1^T=0$.
\item $||\Pi^{1/2}dy_1||_2 = d||\Pi^{1/2}y_1||_2 = 1$ for a well chosen values of $d$. We know that $d \ge 1$ because
\begin{equation*}
|| \Pi^{1/2} y_1|| \le p_{\rm max}^{-1/2} ||y_1||_2 = 1, \quad \text{since } ||y_1||_2 = p_{\rm max}^{-1/2}.
\end{equation*}
\end{itemize}
The objective values achieved by the feasible solution $dy_1$ is $d^2y_1^TLt_1$. Dince $d \ge 1$, we have tat $d^2 y_1^T L y_1 \ge y_1^T L y_1$.
\end{proof}

\section{Conclusion}

In this work, we analyzed the convergence properties of the Value Iteration algorithm using a new line of analysis based on absolute probability sequences. While previous studies have primarily focused on convergence in terms of the infinity norm, we established convergence in the weighted $L^2$ norm, providing a refined perspective on the algorithm’s behavior. 

This analysis offers a complementary viewpoint on the convergence of Value Iteration, which may be useful in further studies on the geometric properties of MDP algorithms. Future work may explore extensions to other iterative methods or examine how these findings apply in broader reinforcement learning settings.

\bibliographystyle{plainnat}
\bibliography{main}

\begin{thebibliography}{8}
\providecommand{\natexlab}[1]{#1}
\providecommand{\url}[1]{\texttt{#1}}
\expandafter\ifx\csname urlstyle\endcsname\relax
  \providecommand{\doi}[1]{doi: #1}\else
  \providecommand{\doi}{doi: \begingroup \urlstyle{rm}\Url}\fi

\bibitem[Bellman(1957)]{bellmandp}
R.~Bellman.
\newblock \emph{Dynamic Programming.}
\newblock Dover Publications, 1957.

\bibitem[Feinberg and Huang(2014)]{feinberg2014value}
Eugene~A Feinberg and Jefferson Huang.
\newblock The value iteration algorithm is not strongly polynomial for discounted dynamic programming.
\newblock \emph{Operations Research Letters}, 42\penalty0 (2):\penalty0 130--131, 2014.

\bibitem[Howard(1960)]{howard1960dynamic}
Ronald~A Howard.
\newblock \emph{Dynamic programming and markov processes.}
\newblock John Wiley, 1960.

\bibitem[Kolmogoroff(1936)]{kolmogoroff1936theorie}
Andrei Kolmogoroff.
\newblock Zur theorie der markoffschen ketten.
\newblock \emph{Mathematische Annalen}, 112\penalty0 (1):\penalty0 155--160, 1936.

\bibitem[Pattabhiraman(1974)]{Pattabhiraman_1974}
M.~V. Pattabhiraman.
\newblock The generalized rayleigh quotient.
\newblock \emph{Canadian Mathematical Bulletin}, 17\penalty0 (2):\penalty0 251–256, 1974.
\newblock \doi{10.4153/CMB-1974-049-4}.

\bibitem[Puterman(1990)]{puterman1990markov}
Martin~L Puterman.
\newblock Markov decision processes.
\newblock \emph{Handbooks in operations research and management science}, 2:\penalty0 331--434, 1990.

\bibitem[Puterman(2014)]{puterman2014markov}
Martin~L Puterman.
\newblock \emph{Markov decision processes: discrete stochastic dynamic programming}.
\newblock John Wiley \& Sons, 2014.

\bibitem[Sutton and Barto(2018)]{sutton2018reinforcement}
Richard~S Sutton and Andrew~G Barto.
\newblock \emph{Reinforcement learning: An introduction}.
\newblock MIT press, 2018.

\end{thebibliography}

\end{document}